\documentclass[letterpaper, 10 pt, conference]{ieeeconf}  

\usepackage[pdftex,dvipsnames]{xcolor}  
\usepackage[colorinlistoftodos,prependcaption,textsize=tiny]{todonotes}

\IEEEoverridecommandlockouts                              
\usepackage{amsthm}
\usepackage{graphics} 
\usepackage{epsfig} 
\usepackage{times} 
\usepackage{amsmath} 
\usepackage{amssymb}  

\usepackage{bm}
\usepackage{xcolor}
\usepackage{graphicx}
\usepackage[noend]{algpseudocode}
\usepackage{booktabs}
\usepackage{siunitx}
\usepackage{hyperref}
\usepackage{courier}

\usepackage{caption}
\usepackage{lipsum}
\usepackage{multirow}
\usepackage{diagbox} 
\usepackage{subcaption}

\usepackage{array}
\usepackage{arydshln}
\usepackage{amsmath,amsfonts}
\usepackage[ruled,linesnumbered]{algorithm2e}
\usepackage[noend]{algpseudocode}

\usepackage{array}
\usepackage{textcomp}
\usepackage{stfloats}
\usepackage{url}
\usepackage{verbatim}
\usepackage{graphicx}

\usepackage{amsthm}
\usepackage{graphics} 
\usepackage{epsfig} 
\usepackage{times} 
\usepackage{amsmath} 
\usepackage{amssymb}  

\usepackage[maxbibnames=99,giveninits=true]{biblatex}
\usepackage{bm}
\usepackage{xcolor}
\usepackage{graphicx}
\usepackage{booktabs}
\usepackage{siunitx}
\usepackage{hyperref}
\usepackage{courier}

\usepackage{subcaption}
\usepackage{caption}
\usepackage{lipsum}
\usepackage{multirow}
\usepackage{diagbox} 

\usepackage{array}
\usepackage{arydshln}
\newcommand{\R}{\mathbb{R}}

\newcommand{\X}{\mathcal{X}_k}

\newcommand{\tx}{\tilde{x}}

\newtheorem{definition}{\bf Definition}

\newtheorem{theorem}{\bf Theorem}
\newtheorem{lemma}{\bf Lemma}

\usepackage[normalem]{ulem}

\begin{document}
\title{
Robustness Analysis of Classification Using Recurrent Neural Networks with Perturbed Sequential Input  
}

\author{Guangyi Liu, Arash Amini, Martin Tak\'a\v{c}, and Nader Motee 
\thanks{G.L., A.A., and N.M. are with the Department of Mechanical Engineering and Mechanics, Lehigh University, Bethlehem, PA 18015, USA {\tt\small \{gliu,a.amini,motee\}@lehigh.edu}.\endgraf
M.T. is with the Machine Learning Department at Mohamed bin Zayed University of Artificial Intelligence (MBZUAI), 
Masdar City,
Abu Dhabi,
United Arab Emirates
{\tt\small \{takac.mt\}@gmail.com}.
}
}

\maketitle



\begin{abstract}
For a given stable recurrent neural network (RNN) that is trained to perform a classification task using sequential inputs, we quantify explicit robustness bounds as a function of trainable weight matrices. The sequential inputs can be perturbed in various ways, e.g., streaming images can be deformed due to robot motion or imperfect camera lens. Using the notion of the Voronoi diagram and Lipschitz properties of stable RNNs, we provide a thorough analysis and characterize the maximum allowable perturbations while guaranteeing the full accuracy of the classification task. We illustrate and validate our theoretical results using a map dataset with clouds as well as the MNIST dataset.  
\end{abstract}


\section{Introduction}

Real-time perception and classification have been among the most exciting topics in computer vision-based and machine learning-based robotic applications. However, in the most perception-based applications, the uncertainties and perturbations prevail in every section of a learned model from the input perturbation \cite{de2001robust,ramesh1992random} to the numerical error \cite{solomon2015numerical}. In order to ensure reliable performance for these applications, some inevitable questions need to be answered: (i) If input perturbation exists, does the learned framework still exhibit robust performance? (ii) What conditions does the learned model needs to satisfy to achieve robustness? (iii) Does there exist characteristics that can measure the robustness of the learned model? Finding the answer to these questions will significantly facilitate tackling perception-based problems since most learned models suffer from the fragility of the input perturbations \cite{tang2017precision}.

In many applications involving area coverage, consensus, map classification, rendezvous \cite{bock2007xv}, the robot can only sample the localized information of the environment \cite{sunderhauf2018limits}, i.e., partially observation, at each time step. In order to acquire adequate observations, the robot could consider traversing the environment while collecting local information as sequential data and attempting to learn their inter-correlation with recurrent neural networks. The same type of framework that uses localized observations and recurrent neural networks to learn the image and map classification is illustrated in our previous works \cite{mousavi2019multi, mousavi2019layered, liu2020distributed, liu2021classification}, and shows promising performance in various scenarios.

In this paper, we consider the map (image) classification problem with the sampled sequential images as a motivational example, which is illustrated in Fig. \ref{fig:intro}. Instead of treating the learned classification model as a black box and solely focusing on its performance, we consider how it will perform with perturbed input and analyze its robustness in terms of its learned weights from the neural networks. The robustness analysis is achieved by considering the stability properties of the recurrent models that learn the interconnection of past observations and quantifying the classification criterion via the Voronoi partitioning to obtain the robustness conditions. 

\begin{figure}[t]
    \centering
	\includegraphics[width=\linewidth, height = 5cm]{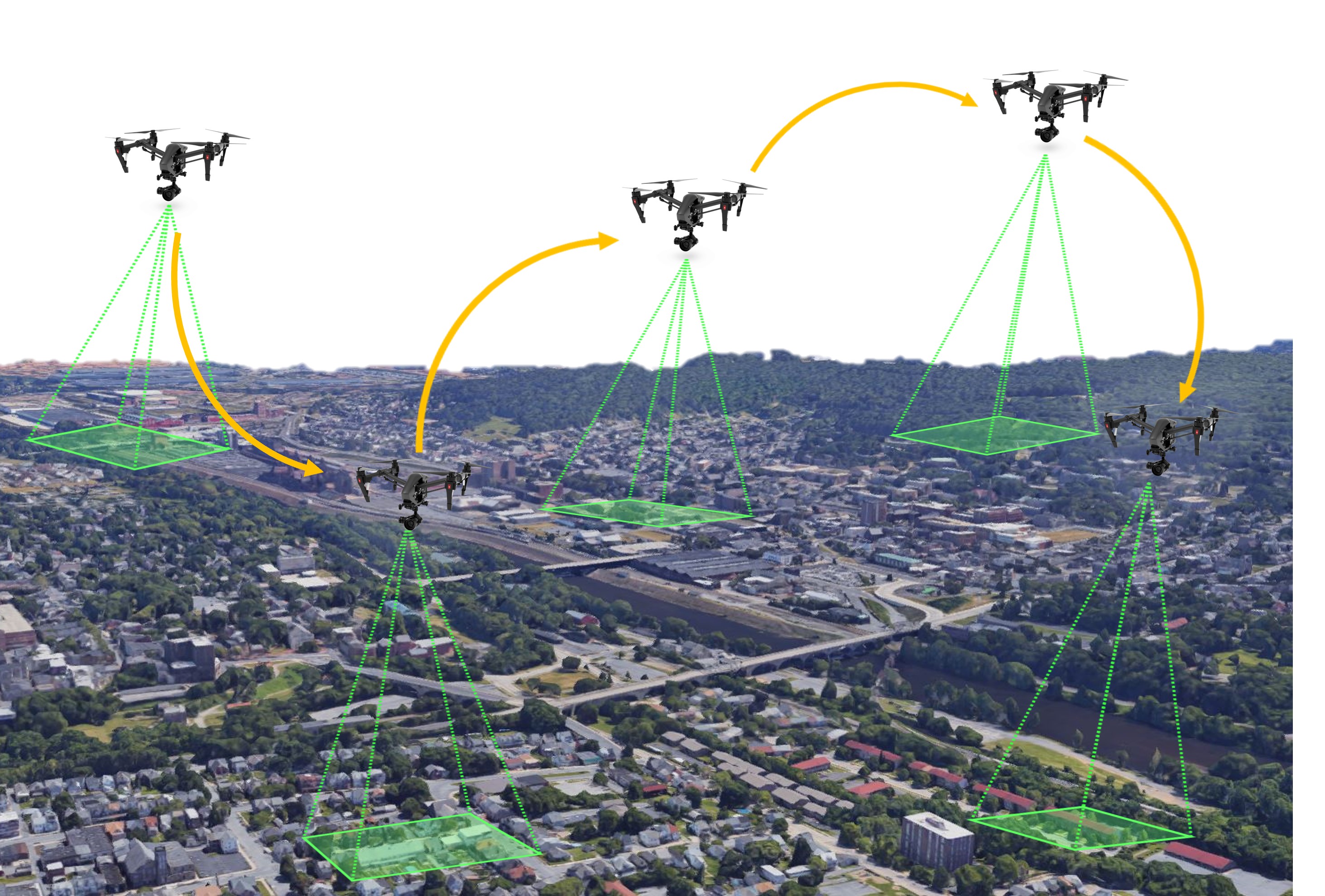}
	\caption{The above figure depicts the motivational application of map classification with the aerial robot. The robot traverses the environment and collects the visual inputs as a sequence (order denoted by arrows).}
    \label{fig:intro}
\end{figure}

The robustness of a classification model can be considered as not miss-classifying under the adversary attacks or perturbations \cite{lowd2005adversarial}, and recent research has made progress on investigating and ensuring the robustness of neural network models under the adversarial attack \cite{xu2018structured,carlini2019evaluating}. Our work exhibits the novelty and differences to these works as: Instead of considering the Boolean classifier, we propose a robustness analysis of multi-class classifier \cite{matyasko2018improved}, in which a novel representation with Voronoi diagram \cite{breitenmoser2010voronoi} is used to construct a quantifiable classification criterion. To ensure the robustness of a learned model, we implement the similar ideas of using expert demonstrations \cite{robey2020learning}, for which we require the perturbed result to stay close to the nominal results, see \S \ref{sec:convergence}. To evaluate the deviation caused by the perturbation at the output stage, we seek the boundedness for every section of the classification model. For instance, inspired by the recent work on constructing a bound for neural network models, to name a few:  \cite{miller2018stable,ko2019popqorn,xu2018structured}, we seek similar boundedness for RNNs, which will provide an error estimate when the statistics of the input perturbation is available \cite{amini2021robust}.  

{\it Our Contributions: }
In this paper, we construct a formal approach to analyze the robustness of the recurrent multi-label classification framework with localized sequential inputs when there exist input perturbations. Furthermore, we also quantify the multi-label classification criterion that uses the $\arg \max$ function. Our analysis shows that a robust classification with full accuracy is guaranteed when the RNNs are stable and the input perturbation is below the maximum allowable deviation. These results motivate us to turn our research efforts to explore further how the classification robustness can show its effect in a closed-loop model where the robot can select its sampling routine based on its observations. 

The rest of the paper is organized as follows. In \S \ref{sec:problem-statement}, we introduce the problem setting and the classification model. The possible origins of perturbations are illustrated in \S \ref{sec:perturbation}. The error estimates and the boundedness of stable RNNs are presented in \S \ref{sec:stable}. Our main result is presented in \S \ref{sec:convergence}, where the quantifiable classification criterion and robust classification conditions are presented. The theoretical findings are validated in \S \ref{sec:case-study} by simulations in both MNIST dataset \cite{lecun1998gradient} and Campus Map dataset \cite{liu2021classification}.

\section{Mathematical Notations}     

The $n-$dimensional Euclidean space with elements $\bm{z} = [z_1, \dots, z_n]^T$ is denoted by $\R^n$, where $\mathbb{R}_{+}$ will denote the positive orthant of $\R^n$. 
The set of standard Euclidean basis for  $\mathbb{R}^{n}$ is represented by $\{\bm{e}_1, \dots, \bm{e}_n\}$. We denote the $n \times n$ identity matrix as $I$ and the vector of all ones as $\bm{1}_n$, respectively. The $i$'th element of a vector $x$ is shown by $x_{i}$ and the $i$'th row of a matrix  $A$ is represented by $(A)_{i}$. The induced matrix norm  by vector norm $\|\hspace{0.05cm}.\hspace{0.05cm}\|$ is also shown by $\|A\|$ \cite{rudin1987real}. Let us define the collection of all feasible probability vectors $p$ \cite{li2015stationary} as {$\mathcal{P}_m = \{p \in \R_{+}^{m} ~|~ p^T \bm{1}_m = 1 \}$}. For a sequence of vectors $x=(x(1), x(2), \ldots, x(T))$, the $\ell_{\infty}$-norm of $x$ is defined by    
\begin{equation}    \label{eq:ell_infty}
        \|x\|_{\ell_{\infty}} = \max_{t \in\{ 1,\dots,T\}}  \|x(t)\|.
    \end{equation}

\section{Problem Statement}     \label{sec:problem-statement}
\begin{figure*}[t]
    \centering
	\includegraphics[width=0.9\linewidth]{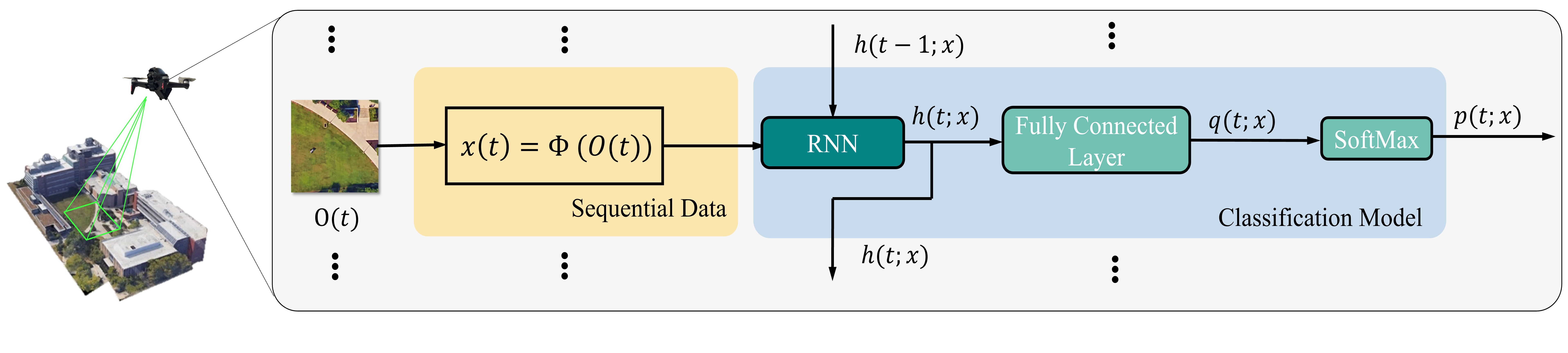}
	\caption{This figure shows the image (map) classification model at time $t$. }
    \label{fig:diagram}
\end{figure*}

Suppose there exist $m$ unique pre-labeled environments for classification purposes. A robot is deployed into the environment to collect samples for classification. The sample, e.g., a vector that contains multiple states of the environment, from data sequence $x$. The data sequence $x$ with length $T$ can be represented in terms of its components as $x = (x(1), x(2), \dots, x(T))$. Let us denote by $\X'$ the set which contains adequate sampled data sequences $x$ from the $k$'th environment, i.e., the training set. Our proposed classifier, depicted in Fig. \ref{fig:diagram}, is trained with $\X'$ for all $k = 1, ..., m$ to classify the label of $x$, which is sampled from an environment with unknown labels. In the classifier, each individual sample of $x$ are fed recurrently to a RNN model \cite{amini2021robust}, whose dynamics can be represented in a compact form by
\begin{align}   \label{eq:LSTM}
    h(t; x) = F \left( h(t-1; x), x(t) \right),
\end{align}
where $h(t;x) \in \R^b$ is the state and $x(t) \in \R^a$ is the input. The (history) state $h(t;x)$ memorizes all the past information about inputs $x$ up to time $t$. The terminal state $h(T;x)$ is used for classification by passing it through a fully-connected layer
\begin{align}\label{eq:q}
    q(T;x) = W_c \: h(T;x) + b_c
\end{align}
and a $\texttt{Softmax}$ function
\begin{align}\label{eq:p}
    p(T;x) = \texttt{Softmax} \left( q(T;x) \right),
\end{align}
where the weight matrix $W_c \in \mathbb{R}^{m \times b}$ and the bias vector $b_c \in \mathbb{R}^{m}$ are trainable. The belief vector $p(T;x)$ is utilized to represent the classification result from the sequence $x$.

Our {\it objective} is to provide a thorough robustness analysis of the classification model using stable RNNs and  quantify their robustness bounds in terms of their trainable weight matrices.


\section{Origins of Perturbations}      \label{sec:perturbation}
\begin{figure}[t]
    \centering
	\includegraphics[width=\linewidth]{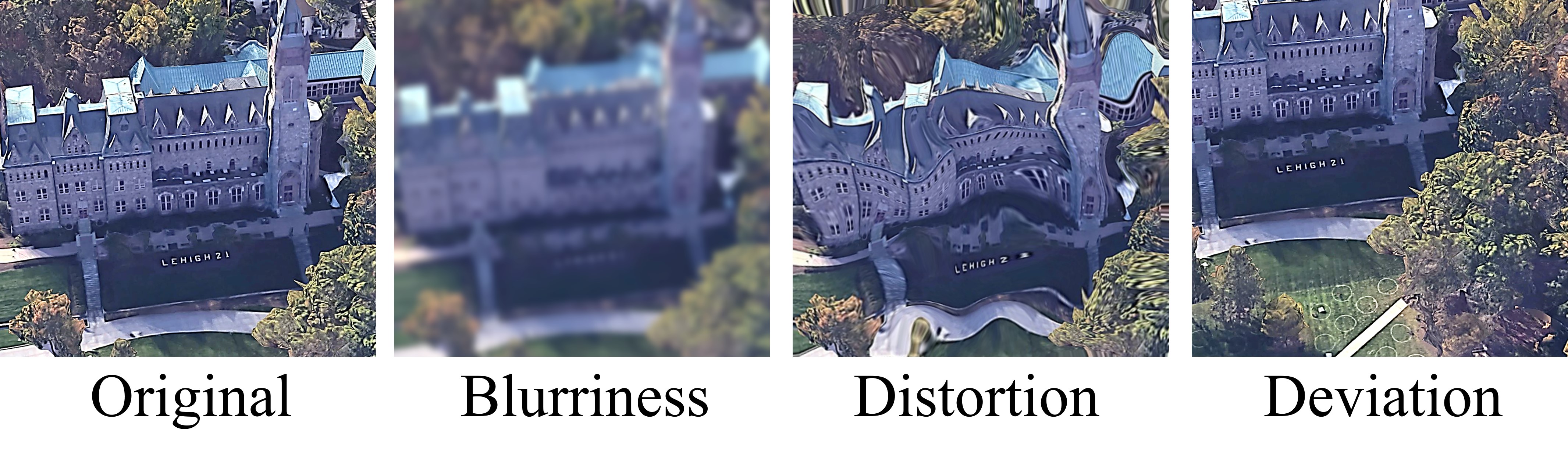}
	\caption{An input can be perturbed in various ways. This figures illustrate some possible cases for the map classification task.}
    \label{fig:origins_perturbation}
\end{figure}

There exist several ways by which data can be collected as a sequence from the environment; for example, an agent can navigate in an environment and take localized observations  \cite{liu2020distributed} or a camera with a fixed location can capture images from a time-varying scene \cite{huang1981image}. In most real-world applications, such raw observations are pre-processed, e.g., by neural networks, and the data sequence already carries relevant features of the observed raw data. Let us assume that this pre-processing can be modeled by a nonlinear map
\begin{equation*}
    x(t) = \Phi \left( O (t) \right),
\end{equation*}
where $O(t) \in \mathcal{O}$ denotes the raw observation sampled at time $t$, and $\mathcal{O}$ is the space of all observables. 

Perturbations can affect the data quality through several possible sources during the sampling process. The observer may deviate from its initially planned sampling routine due to dynamic noise in its motion planning \cite{du2011robot} and take (slightly) deviated samples from nearby scenes in the environment. The raw observations may lose quality due to the environmental noise \cite{boie1992analysis}, e.g., change in light intensity, cloudiness, and blurriness. The raw observations may also experience various types of deformation, e.g., camera rotation or distortion \cite{tang2017precision}. Fig. \ref{fig:origins_perturbation} depicts some  of these perturbations. The effect of uncertainty in all these cases can be modeled by 
\begin{equation} \label{eq:perturbation}
\tx(t) = \Phi \left( \tau \left(O (t) \right)+\xi(t) \right),
\end{equation}
where $\tx(t)$ represents the perturbed data, $\xi: \R \rightarrow \mathcal{O}$ is an additive bounded stochastic noise or a bounded deterministic disturbance, and $\tau: \mathcal{O} \rightarrow \mathcal{O}$ is a deformation map that models sample deformation due to sensor movements (e.g., translation and rotation, and scaling).

\section{Stable Recurrent Neural Networks and Their Error Estimates}  \label{sec:stable}

In order to analyze the robustness of the classification model, we will evaluate how the input perturbation is affecting the classification result. Let us first identify what input sequences can generate correct classifications. Recall that the model is trained with sequences from $\X'$ for all $k = 1,...,m$, and not every $x \in \X'$ will carry enough information to reveal the environment (e.g., a sequence consists of repeating scenes). Hence, only a subset of the sequences from the training set may generate the correct classification. We represent those data sequences that generate correct classifications by $\X^*$, where $\X^* \subseteq \X'$. This relation is depicted in Fig. \ref{fig:x_relation}.   

To measure how the input perturbation will affect the classification, let us consider a nominal sequence $x \in \X^*$ and its corresponding perturbed sequence $\tx \in \X$ obtained from \eqref{eq:perturbation}. The set $\X$ is the space of all possible sequences for the $k$'th class. In this section, we aim to evaluate the deviation generated by $x$ and $\tx$ in terms of belief vectors, i.e., $p(T;x)$ and $p(T;\tx)$. The first step is to consider the deviation generated at the output of the RNN, i.e., $h(t; x)$ and $h(t; \tx)$. 

In order to demonstrate our next result, let us introduce the concept of stable RNNs and its Lipschitz property. A stable RNN \cite{miller2018stable} provides the boundedness to its output when the input vectors are identical, i.e., $x(t) = \tx(t)$.

\begin{figure}
    \centering
	\includegraphics[width=0.4\linewidth]{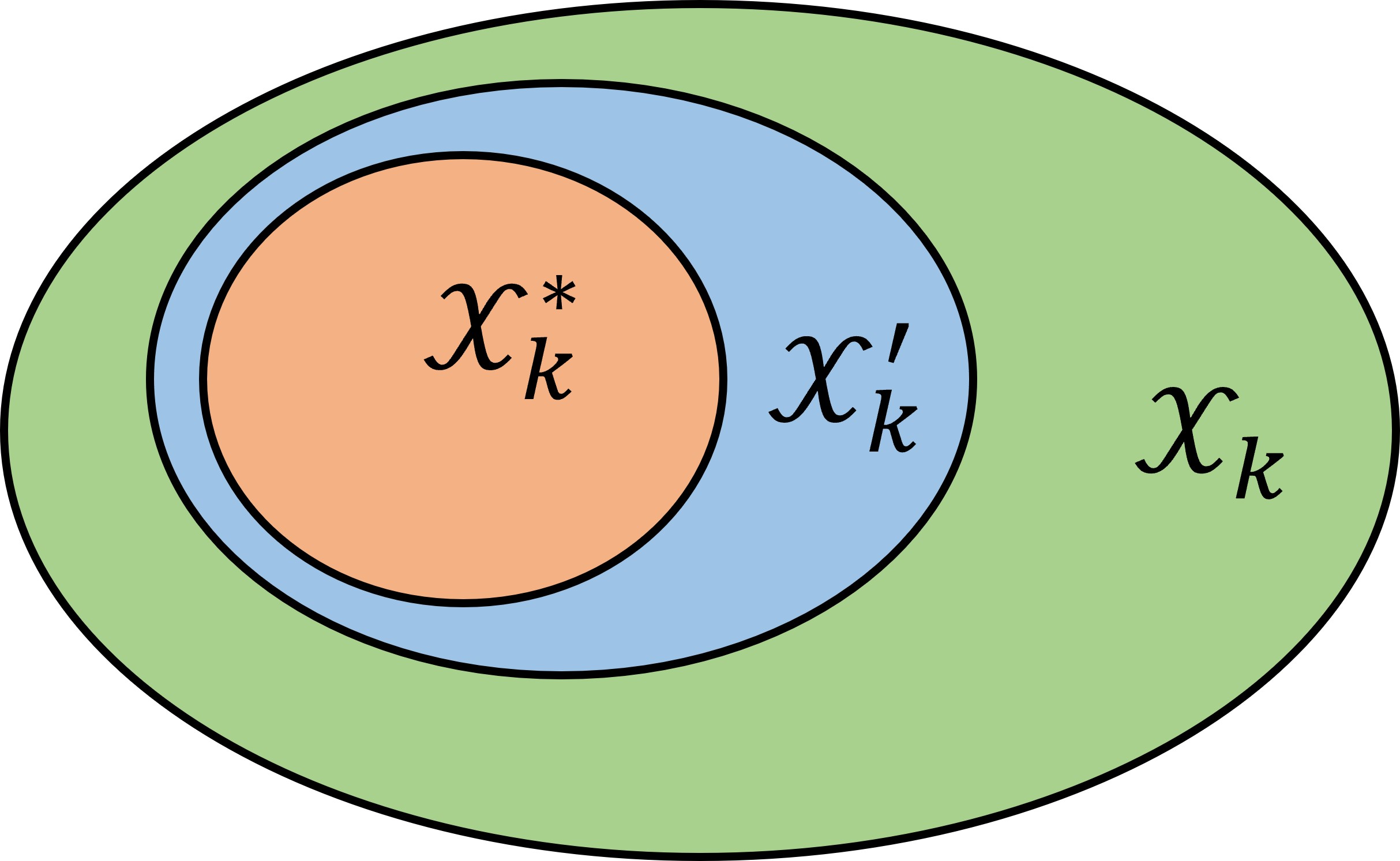}
	\caption{This figure depicts the relation of the space of all possible sequences $\X$, the training set $\X'$, and the nominal set $\X^*$, i.e., $\X^* \subseteq \X' \subseteq \X$.}
    \label{fig:x_relation}
\end{figure}

\begin{definition}
    A recurrent neural network model \eqref{eq:LSTM} is stable (contractive) if there exists a constant $\lambda \in (0,1)$ such that for any $h(t-1; x),  h(t-1; \tx) \in \R^{b}$,
    \begin{equation}    \label{eq:contractive}
    	\|h(t; x) - h(t; \tx)\| \leq \lambda \, \|h(t-1; x) - h(t-1; \tx)\|,
    \end{equation}
    where $h(t;x) = F \big( h(t-1; x), x(t) \big)$ and $h(t; \tx) = F \big( h(t-1; \tx), x(t) \big)$. 
\end{definition}

It is also known that the RNN model is Lipschitz continuous with respect to its input \cite{miller2018stable, ko2019popqorn}. Then, for every two sequences $x$ and $\tx$ with identical history state at $t-1$, i.e., $h(t-1;x) = h(t-1;\tx)$, one has
\begin{align}    \label{eq:lip} 
    \|h(t;x) - h(t;\tx)\| \leq \kappa ~ \|x(t) - \tx(t)\|,
\end{align}
where $\kappa \in \mathbb{R}_{+}$ is the real Lipschitz constant, $h(t;x) = F \big( h(t-1; x), x(t) \big)$, and $h(t; \tx) = F \big( h(t-1; x), \tx(t) \big)$. Then, the following result shows the upper bound for the deviation between the nominal belief $p(T;x)$ to the perturbed belief $p(T;\tx)$ given a stable RNN model.

\begin{theorem}     \label{thm:1}
    For the classification model with a stable RNN model, one has
    \begin{equation*}
        \|p(T;x) -p(T;\tx)\| \leq \eta \, \|x - \tx\|_{\ell_{\infty}}
    \end{equation*}
    for every two sequences $x \in \X^*$ and $\tx \in \X$, where 
    \begin{equation}    \label{eq:eta}
        \eta=\frac{ \kappa ~ \|W_c\|}{(1-\lambda) \sqrt{m}}.
    \end{equation}
\end{theorem}

\begin{proof}
    The initial hidden states are set to $h(0;x) = 0$ and $h(0;\tx) = 0$ in both training and testing stage. Considering the Lipschitz continuity \eqref{eq:lip}, the history states generated by sequences $x$ and $\tx$ at $t = 1$ follows
    \begin{equation} \label{eq:lip-1}
        \begin{aligned} 
            \|h(1;x) - h(1;\tx)\| \leq \kappa \|x(1) - \tx(1)\|.
        \end{aligned}
    \end{equation}
    For the next time step $t = 2$, we can obtain \eqref{eq:h2} by using the triangle inequality
    \begin{multline}\label{eq:h2}
        \|h(2;x) -  h(2;\tx)\| = \|h(2;x) - F \big(h(1;x),\tx(2)\big) \\+ F \big(h(1;x),\tx(2)\big) - h(2;\tx)\|\\
    	\leq \|F \big(h(1;x),x(2) \big) - F \big(h(1;x),\tx(2) \big)\|\\
    	 + \|F \big(h(1;x),\tx(2) \big) - F \big(h(1;\tx),\tx(2) \big)\|.
    \end{multline} 
    The first half of \eqref{eq:h2} can be bounded using \eqref{eq:lip},
    \begin{align*}
        \|F \big(h(1;x),x(2) \big) -F \big( h(1;x),\tx(2) \big)\| &\leq \kappa ~ \|x(2) - \tx(2)\|\\
        & \leq \kappa ~ \|x - \tx\|_{\ell_\infty}.
    \end{align*}
    The second half of \eqref{eq:h2} can be bounded using \eqref{eq:contractive} and then \eqref{eq:lip-1},
    \begin{align*}
        &\|F (h(1;x),\tx(2)) -F (h(1;\tx),\tx(2))\|\\ 
        &\hspace{2cm}\leq \lambda~\|h(1;x) - h(1;\tx)\|\\
        &\hspace{2cm}\leq \lambda \kappa ~\|x(1) - \tx(1)\| \leq \lambda \kappa ~\|x - \tx\|_{\ell_\infty}.
    \end{align*}
    Summarizing the above inequalities, the deviation of history states at $t = 2$ obtains the following upper bound,
    \begin{align*}
        \|h(2;x) - h(2;\tx)\|  \leq (1+\lambda) \, \kappa \, \|x - \tx\|_{\ell_\infty}.
    \end{align*}
    Repeating the above steps up to $t$, the deviation of $\|h(t;x) - h(t;\tx)\|$ is upper bounded by
    \begin{align*}
    	\|h(t;x) - h(t;\tx)\| &\leq (1 +\lambda + ... + \lambda^{t-1} )\kappa ~ \|x-\tx\|_{\ell_{\infty}}\\
    	&\leq \frac{1-\lambda^{t}}{1-\lambda} \kappa ~ \|x-\tx\|_{\ell_{\infty}}.
    \end{align*}
    Given that $0< \lambda <1$, the terminal deviation at $T$ obtains the following boundedness,
    \begin{align}   \label{eq:h-diff} 
        \|h(T;x) - h(T;\tx)\| \leq \frac{\kappa}{1-\lambda} ~ \|x-\tx\|_{\ell_{\infty}}.
    \end{align}
    The similar bounds can be obtained for the fully connected layer \eqref{eq:q},
    \begin{align*}
        \|q(T;x) - q(T;\tx) \| \leq \|W_c\| \frac{\kappa}{1-\lambda} ~ \|x-\tx\|_{\ell_{\infty}}.
    \end{align*}
    Then, the results follows immediately by considering the fact that the $\texttt{Softmax}$ function is $1/\sqrt{m}$ Lipshcitz continuous \cite{2021011} with respect to its input $q \in \R^m$.
\end{proof}

The above theorem asserts that, in the classifier, if the RNN model is stable, the terminal belief difference is bounded by the maximum deviation along the input sequences. This result provides the knowledge of the classifier's robustness by identifying an upper bound for the error generated by the perturbed and the nominal data. 

\section{Classification Criterion and Convergence Conditions} \label{sec:convergence}

\begin{figure}[t]
    \centering
	\includegraphics[width=\linewidth]{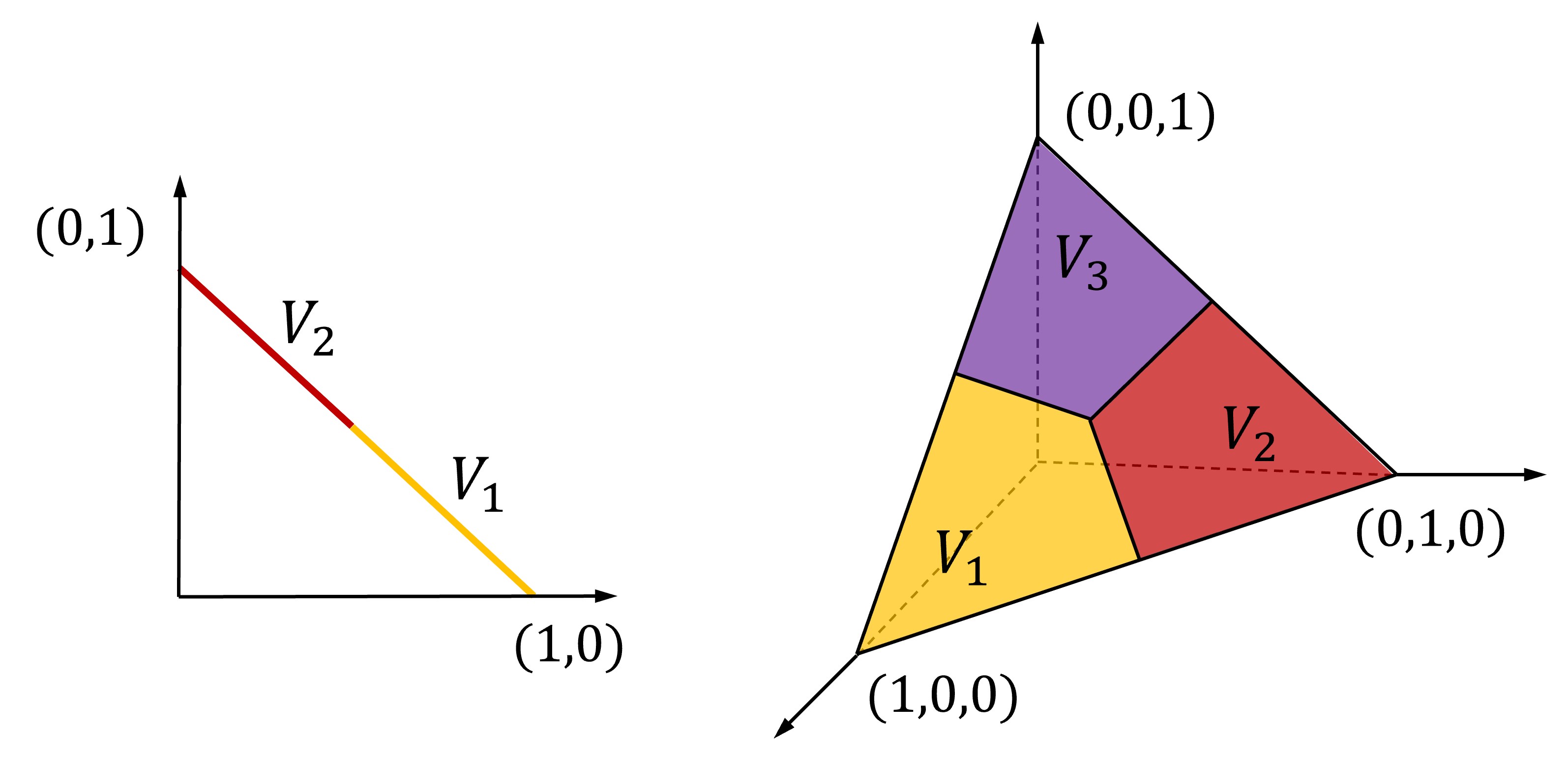}
	\caption{This figure depicts the Voronoi partition of the classification sets $\mathcal{P}_m$ for both $m=2$ (left) and $m=3$ (right).}
    \label{fig:voronoi}
\end{figure}

In this part, we aim to represent the $\arg\max$ classification criterion in a quantifiable manner and use the previously obtained deviation bounds in belief vectors to perform the robustness analysis.

The classification process with $x$ is accomplished by identifying $\arg \max p(T;x)$ as the class label, i.e., it will conclude as the $k$'th class if and only if $p(T;x)_k > p(T;x)_j$ for all $j = 1, \dots, m$ and $j \neq k$. This criterion can be explicitly represented via the Voronoi partitioning \cite{klein1988abstract}. Let us denote by $V_1, V_2, \dots, V_m$ the Voronoi partition of the probability vector space $\mathcal{P}_m$, i.e.,
\begin{equation}    \label{eq:v_k}
    V_k = \left\{p \in \mathcal{P}_m ~|~ p_k > p_j, ~  \forall j \neq k \right\}. 
\end{equation}
Some examples of the Voronoi partitioning is shown in Fig. \ref{fig:voronoi}. The above partition is equivalent to the $\arg\max$ classification criterion, i.e., $\arg \max p(T;x) = p(T;x)_k$ if and only if $p(T;x) \in V_k$.

The next step is to convert the deviation $\|p(T;x) -p(T;\tx)\|$ into a Boolean classification result, i.e., ``true" or ``false". To reveal our next result, let us introduce the distance function between vectors and sets.

\begin{definition}
    The distance between two vectors $p , p' \in \mathcal{P}_m$ is defined as
    \begin{equation*}
        d(p, p') = \|p - p'\|
    \end{equation*}
    and the distance between a vector and a set is defined as
    \begin{equation*}
        d(p, V_j) = \inf_{p' \in V_j} d(p, p').
    \end{equation*}
\end{definition}

Let us express the $\arg \max$ classification criterion equivalently using the distance function.

\begin{lemma}       \label{lem:criterion}
    A belief vector $p \in \mathcal{P}_m$ will be classified as the $k$'th class if and only if
    \begin{equation}    \label{eq:criterion1}
        d(p, V_k) = 0,
    \end{equation}
    or if and only if 
    \begin{equation}    \label{eq:criterion2}
        d(p, V_j) > 0 \text{ for all } j \neq k.
    \end{equation}
\end{lemma}

\begin{proof}
    Suppose there exists a $p' \in V_k$ such that $d(p,p') = 0$, then one has $p = p'$ and $p \in V_k$. On the other hand, if $p \in V_k$, then we have $d(p, V_k) \leq d(p, p) = 0$ and $d(p, V_k)  = 0$ by definition.
        
    For the equivalence, since $V_k \bigcap V_j = \emptyset$ for all $j \neq k$, one has if $d(p, V_k) = 0$, then $d(p, V_j) > 0$ for any $V_j$. On the other hand, consider the fact that 
    $$
    V_k \subset \mathcal{P}_m \setminus \bigcup_{j \neq k} V_j,
    $$ 
    one has if $d(p, V_j) > 0$, then $p \notin V_j$ and $p \in V_k$ \footnote{The cases of $p$ located on the boundary of the Voronoi partition is considered trivial since its probability of happening is $0$.}.
\end{proof}

The above result introduces a quantifiable classification criterion, which can be used in all classification problems using the $\arg \max$ classifier. Given a belief vector $p \in \mathcal{P}_m$, one can obtain the classification result by checking if \eqref{eq:criterion1} or \eqref{eq:criterion2} are satisfied.

Using Lemma \ref{lem:criterion}, we can establosh a  connection between $\|p(T;x) -p(T;\tx)\|$ and the robustness of the classification model. For a robust classification model, we expect $p(T;\tx)$ to be located within $V_k$. This relation can be validated by comparing $d(p(T;x), V_j)$ with $\|p(T;x) -p(T;\tx)\|$ for all $j = 1,...,m$ and $j \neq k$. In order to accomplish the analysis for an arbitrary perturbed sequence $\tx$, let us introduce the concept of the robustness radius.

\begin{definition}
    For all nominal sequences $x \in \X^*$, let us consider the {\it robustness radius} for the $k'$th class label as 
    \begin{equation}    \label{eq:robust-radius}
        \varepsilon_k = \min_{x \in \X^*, \, j \neq k} d \big( \, p(T;x), V_j \, \big),
    \end{equation}
    where $j = 1, \dots, m$. 
\end{definition}

The robust radius quantifies the minimal distance from a nominal belief vector $p(T;x)$ to the boundary of $V_k$, which enables us determine when the classification with $\tx \in \X$ is robust, i.e., the result is the same with the one generated by some $x \in \X^*$.

\begin{theorem}     \label{thm:2}
    Suppose that the RNN model \eqref{eq:LSTM} is stable. The classification generated by a perturbed sequence $\tx \in \X$ is robust for the $k$'th class, if there exist some $x \in \X^*$ such that
    \begin{equation}    \label{eq:thm2}
        \|x - \tx \|_{\ell_{\infty}}  < \varepsilon_k \, \eta^{-1},
    \end{equation}
    where $\|x - \tx \|_{\ell_{\infty}}$ and $\eta$ are defined in \eqref{eq:eta} and \eqref{eq:ell_infty}.
\end{theorem}

\begin{proof}
    In the view of the $k$'th class, the classification with $\tx \in X$ is robust if 
    \begin{equation} \label{eq:tx_dist}
        d(\,p(T;\tx), V_j\,) = \inf_{p' \in V_j} d(p(T;\tx), p') > 0,
    \end{equation}
    for all $j = 1,\dots,m$ and $j \neq k$. The above quantity obtains a lower bound using the triangle inequality,
    \begin{align*}
        \inf_{p' \in V_j}d(p(T;\tx), p') &\geq  \\
        \inf_{p' \in V_j} & d(p(T;x), p') - d(p(T;x), p(T;\tx)),
    \end{align*}
    for any $\tx \in \X$ and $x \in \X^*$. Then, the inequality \eqref{eq:tx_dist} will be satisfied if 
    \begin{align*}
        \inf_{p' \in V_j} d(\,p(T;x), p'\,) - d(\,p(T;x), p(T;\tx)\,)> 0,
    \end{align*}
    which is equivalent to 
    $  
        \|p(T;x) - p(T;\tx)\| < \varepsilon_k.
    $
    Then, we can conclude by applying Theorem \ref{thm:1} to the inequality above.
\end{proof}

The above theorem states that in the classification model with stable RNNs, a perturbed data sequence $\tx$ will generate a correct classification result if there exist some nominal sequences $x \in \X^*$ that satisfy \eqref{eq:thm2}. In other words, the constant $\varepsilon_k \eta^{-1}$ measures maximal allowed deviation for the input to perform robust classification on the $k$'th class, i.e., for any perturbed sequence $\tx$ that satisfies \eqref{eq:thm2}, the classification result is guaranteed to be correct. In Fig. \ref{fig:robust_radius}, we present the idea of this relation with a simplified example.
We highlight that the above result provides a sufficient condition, so $\tx$ may still generate a correct classification if \eqref{eq:thm2} is not satisfied.

\begin{figure}[t]
    \centering
	\includegraphics[width=\linewidth]{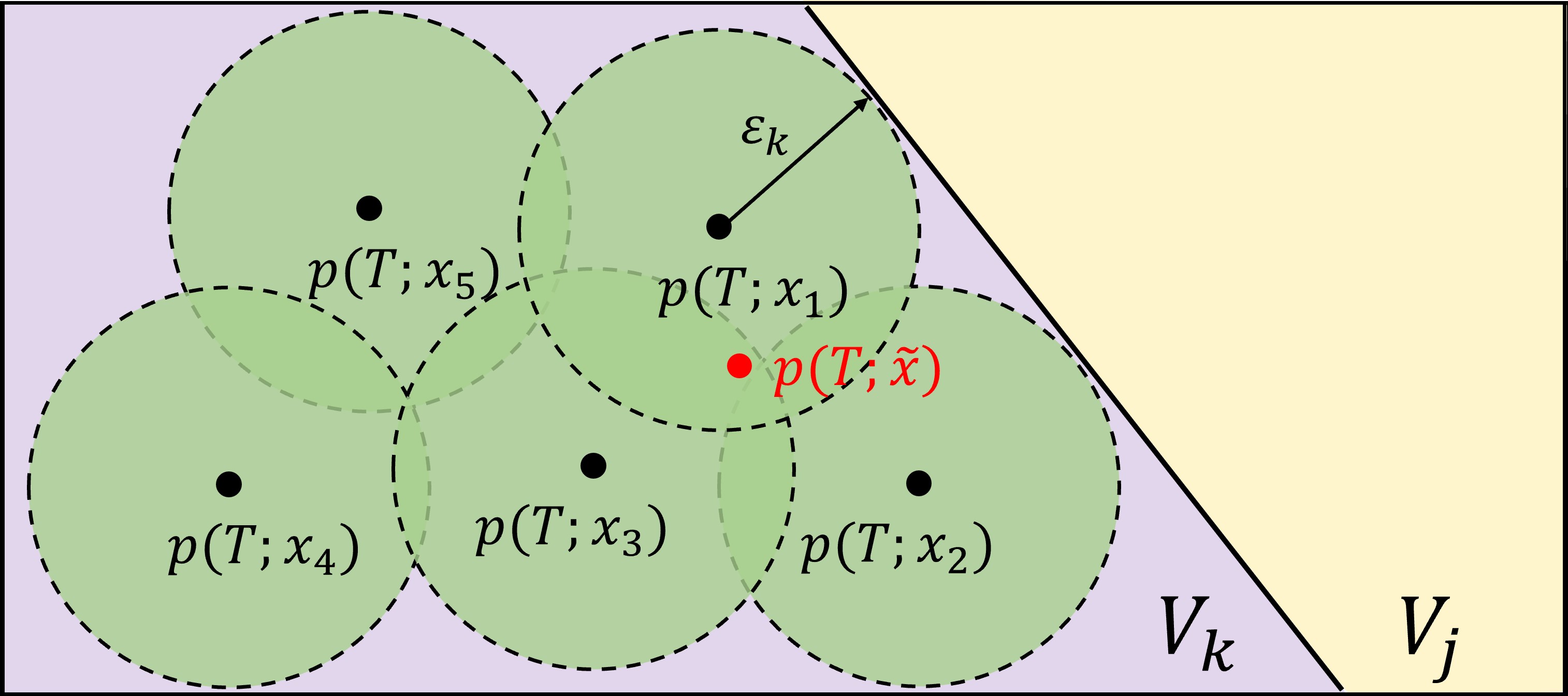}
	\caption{This figure depicts nominal belief vectors in the $k$'th class $p(T;x_1),...,p(T;x_5)$, and the perturbed belief vector $p(T;\tx)$. The model will perform a robust classification if $p(T;\tx)$ is within the distance of $\varepsilon_k$ to some $p(T;x_i)$, where $x_i \in \X^*$.}
    \label{fig:robust_radius}
\end{figure}


\section{Case Study and Simulations}      \label{sec:case-study}
We use the map and image classification \cite{liu2020distributed} as examples to demonstrate and validate our theoretical results. In the case study, it is assumed that an aerial robot with a downward-facing camera aims to classify the underlying image (or map). However, due to the limited sensing capability, the robot can only observe a localized portion of the image at each time step, and it is allowed to traverse the environment to collect partial images as a time-indexed sequence. 

\subsection{Training for Classification} 
We consider the sampling routine for the robot is given and fixed for both training and testing stage \footnote{For each map (or image), we generate five unique paths for a robot to traverse, and they are fixed through the training and testing. However, robots are also capable of planning the path based on their observation, see \cite{liu2021classification, mousavi2019layered}}. Data sequences are generated with a fixed length $T$. In the case study, a VGG-19 \cite{simonyan2014very} model is adopted to process the raw observations $O(t) \in \mathcal{O}$, such that 
$
    x(t) = VGG \big( O(t) \big).
$
The robot also uses a stable LSTM cell \cite{lstm}, which is a special case of RNN, to recurrently process the input. A stable (contractive) LSTM can be learned by introducing the following constraints in the training stage \footnote{In our notation, $\|f\|_{\infty} = \sup_t \|f_{ t \geq 0}\|_{\infty}$ and  $f_t$ is the output of the forget gate of the LSTM cell. We refer to \cite{miller2018stable} for the details.} 
\begin{equation}    \label{eq:constraint}
    \max \left\{\|W_u\|_{\infty}, \|W_o\|_{\infty}, 4 \|W_z\|_{\infty}, \sqrt{\|W_f\|_{\infty}}  \right\}  <  1 - \|f\|_{\infty}.
\end{equation}
When the above constraint is not satisfied during the training, each row of the LSTM weight matrices will be divided by a constant scaling factor $\tau > 1$ after each gradient step until \eqref{eq:constraint} is satisfied.

In the training stage, we evaluate the classification reward with the data sequence $x \in \X'$ by a log-sum-exp (LSE) loss as
$
    r = -LSE \big( e_k, p(T;x) \big),
$
in which $k$ denotes the ground truth label.

\subsection{MNIST Dataset}
\begin{figure}[t]
    \centering
	\includegraphics[width=\linewidth]{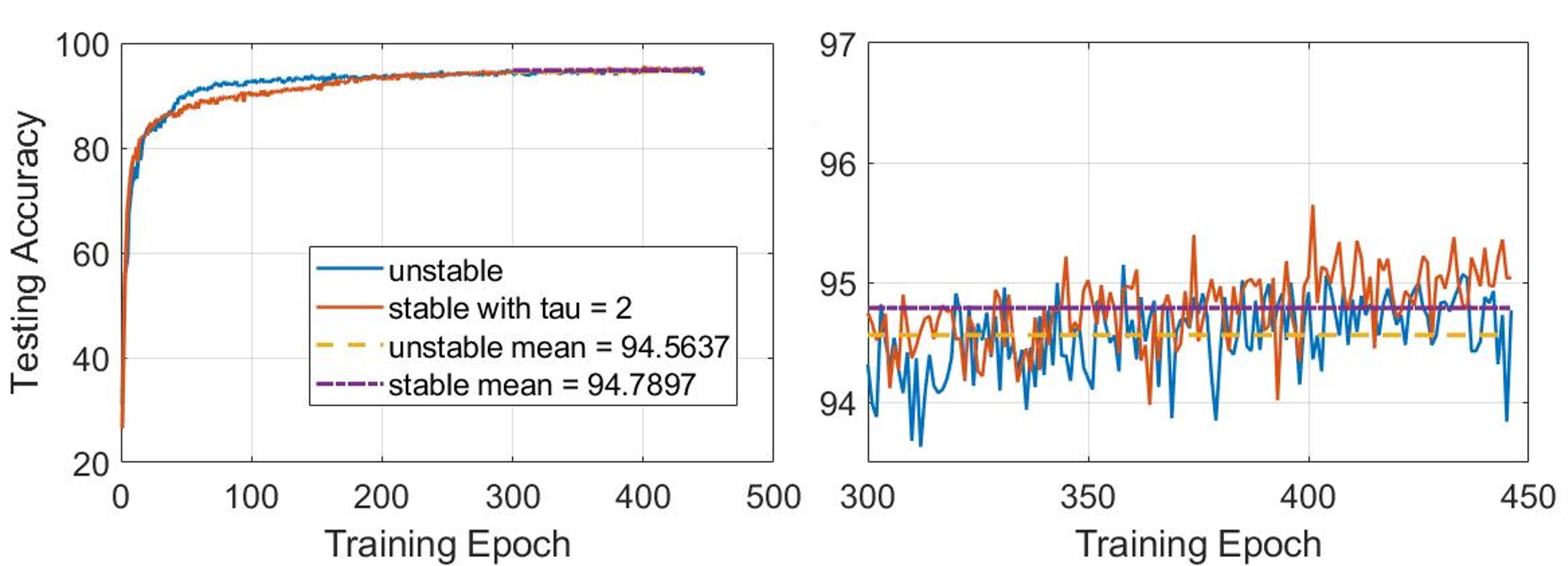}
	\caption{This figure depicts the training progress with and without the stability constraint \eqref{eq:constraint} on the MNIST dataset.}
    \label{fig:loss}
\end{figure}

In the first case study, we consider the robot is traveling over the image from the MNIST dataset \cite{lecun1998gradient, mousavi2019multi}. The training and testing is performed in PyTorch \cite{paszke2017automatic} with ADAM \cite{kingma2014adam} and a learning rate $l_r = 0.0001$. The testing result is presented in Table \ref{table:accuracy}, which is validated over five random seeds.

To establish a benchmark, we train an independent VGG-19 model with the entire image (map) as the input, shown in the last column of Table \ref{table:accuracy}. In the second column, the observation size denotes the maximum possible coverage of the entire image or map. It is shown that a single robot can classify by only revealing a small portion of the environment. It should be emphasized that the performance on the MNIST will reach $98\%$ by using multiple communicating robots \cite{mousavi2019multi}. 

\begin{table*}[t]                           
	\centering	
	\resizebox{\linewidth}{!}{ 
		\begin{tabular}{ l c c c c c}
			\toprule
			Dataset &	Observation Size ($\%$)	&	Stability Constraint & Scaling Factor $\tau$ & Classification Accuracy ($\%$) &  Benchmark (VGG-19) \\
			\midrule
			\multirow{2}{*}{Map Dataset with clouds}        & \multirow{2}{*}{$\leq$7.8}  & No & -  & 77.62 & \multirow{2}{*}{99.43}\\   
		    &   & Yes & 1.01 & 77.23 &           \\ [0.5ex]
			\hdashline\noalign{\vskip 0.5ex}
			\multirow{6}{*}{MNIST}      & \multirow{6}{*}{$\leq$ 68.88} & No & - & 94.56 & \multirow{6}{*}{99.33}          \\   
		    &  & Yes & 1.05 & 94.34 &           \\ 
		    &  & Yes & 1.1 & 94.90 &           \\
		    &  & Yes & 2 & 95.05 &           \\
		    &  & Yes & 4 & 94.20 &         \\
		    &  & Yes & 8 & 94.04 &         \\
			\bottomrule
	\end{tabular} }                        
	\caption{The performance measured with different observation size, stability condition and dataset. }
	\label{table:accuracy}  
\end{table*}

\subsection{Robustness Analysis on the MNIST Dataset}
To reveal how a stable RNN model will affect the performance of the classifier, we test both unstable and stable models with various scaling factors $\tau$ in the MNIST dataset. As shown in both Fig. \ref{fig:loss} and Table \ref{table:accuracy}, a stable model has a comparable performance and sometimes it even outperforms the unstable model for certain values of $\tau$.

\subsubsection{Constant deviation}
To investigate the robustness of the classifier, we first obtain the nominal sequences\footnote{A sequence is called nominal if its resulting classification accuracy is  $100\%$.} from the training set $\X'$ and measure the performance of the model by adding a constant deviation $\xi \in [0,10]$ to the nominal sequences to get
$$
    \tx(t) = x(t) + \xi \bm{1}_a,
$$
where $\bm{1}_a$ is the vector of all ones. This implies that  $\|x - \tx\|_{\ell_{\infty}} =  \xi$. The test is performed on the $5$'th class of the MNIST dataset for the unstable and stable models with $\tau = 1.05$. The result is shown in Fig. \ref{fig:deviation}, and it implies that, for the stable model, if the deviation of the input $\|x - \tx\|_{\ell_{\infty}}$ is less than $\varepsilon_{5} / \eta = 0.0886$, the classification with $\tx$ is guaranteed to be robust, i.e., with a $100\%$ accuracy. This agrees with our theoretical result. However, the unstable model starts to generate wrong predictions before the quantity $\|x - \tx\|_{\ell_{\infty}}$ reaches $0.0886$. 

We highlight that for the stable model the accuracy remains $100\%$ even for some $\|x - \tx\|_{\ell_{\infty}} > \varepsilon_{5} / \eta$. This is because our theoretical result provides a sufficient condition, which is usually conservative. As we observe from the simulations, there is still a chance to get robust classification when \eqref{eq:thm2} is not satisfied. Furthermore, it is  interesting to notice that the unstable model starts to outperform the stable model when $\|x - \tx\|_{\ell_{\infty}} \geq 0.2$. The reason is that  by imposing the stability constraint \eqref{eq:constraint} one only require the stable model to {\it confidently} exhibit the robust classification with the input perturbation less than $\varepsilon_{k} / \eta$,  instead of concerning the performance with $\|x - \tx\|_{\ell_{\infty}} \geq \varepsilon_{k} / \eta$. The stable model provides the confidence of $100\%$ accuracy with all deviations $\|x - \tx\|_{\ell_{\infty}} < 0.0886$, which is not guaranteed for the unstable models, as shown in the shaded area in Fig. \ref{fig:deviation}. This difference is crucial when the robot is performing high-precision tasks, in which any level of mistake is not acceptable. On the other hand, the stable model that focuses on improving the performance with the Gaussian input noise has been proposed and validated in our previous work \cite{amini2021robust}.  

\subsubsection{Variable deviation}
It is also interesting to see how the model will perform with the variable deviations instead of the constant ones. The differences from the testing dataset to the training dataset can be naturally considered variable deviations since the variation of handwritten digits from the training to the testing set can not be modeled as the constant deviation. Hence, we compare the performances of different models on both training and testing sets, see Fig. \ref{fig:perturb} and Table \ref{table:robust}, in which we denote the training set as ``unperturbed" and the testing set as ``perturbed." 

The results of variable deviation show that the performance loss from unperturbed data to the perturbed data is significantly reduced by using a stable model. In addition, the stable model sometimes outperforms the unstable model if a proper scaling factor is selected, e.g., $\tau = 1.1$ and $\tau = 2$. We also highlight that there exists an potential trade-off between the performance loss and the overall performance: As shown in Table \ref{table:robust}, a higher scaling factor $\tau$ usually implies a more stable RNN and less performance loss, but also a weaker overall performance since the stability constraint will potentially drive the model away from the optimal classifier.

\begin{figure*}[t]
    \begin{subfigure}[t]{.32\linewidth}
        \centering
    	\includegraphics[width=\linewidth]{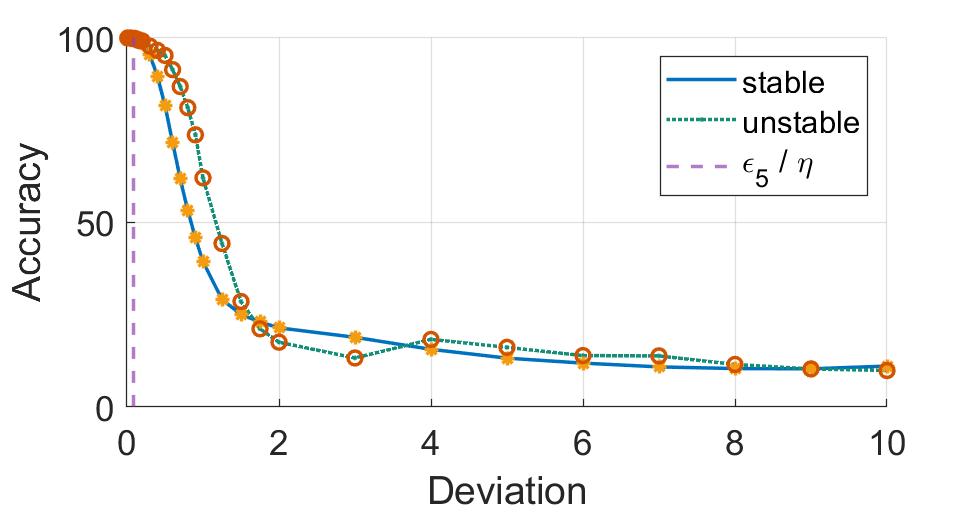}
    \end{subfigure}
    \hfill
    \begin{subfigure}[t]{.32\linewidth}
        \centering
    	\includegraphics[width=\linewidth]{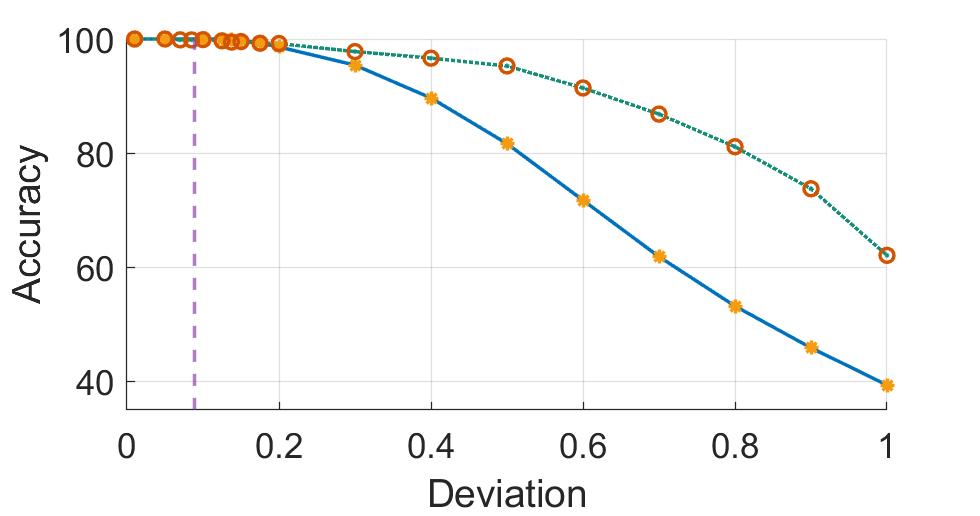}
    \end{subfigure}
    \hfill
    \begin{subfigure}[t]{.32\linewidth}
        \centering
    	\includegraphics[width=\linewidth]{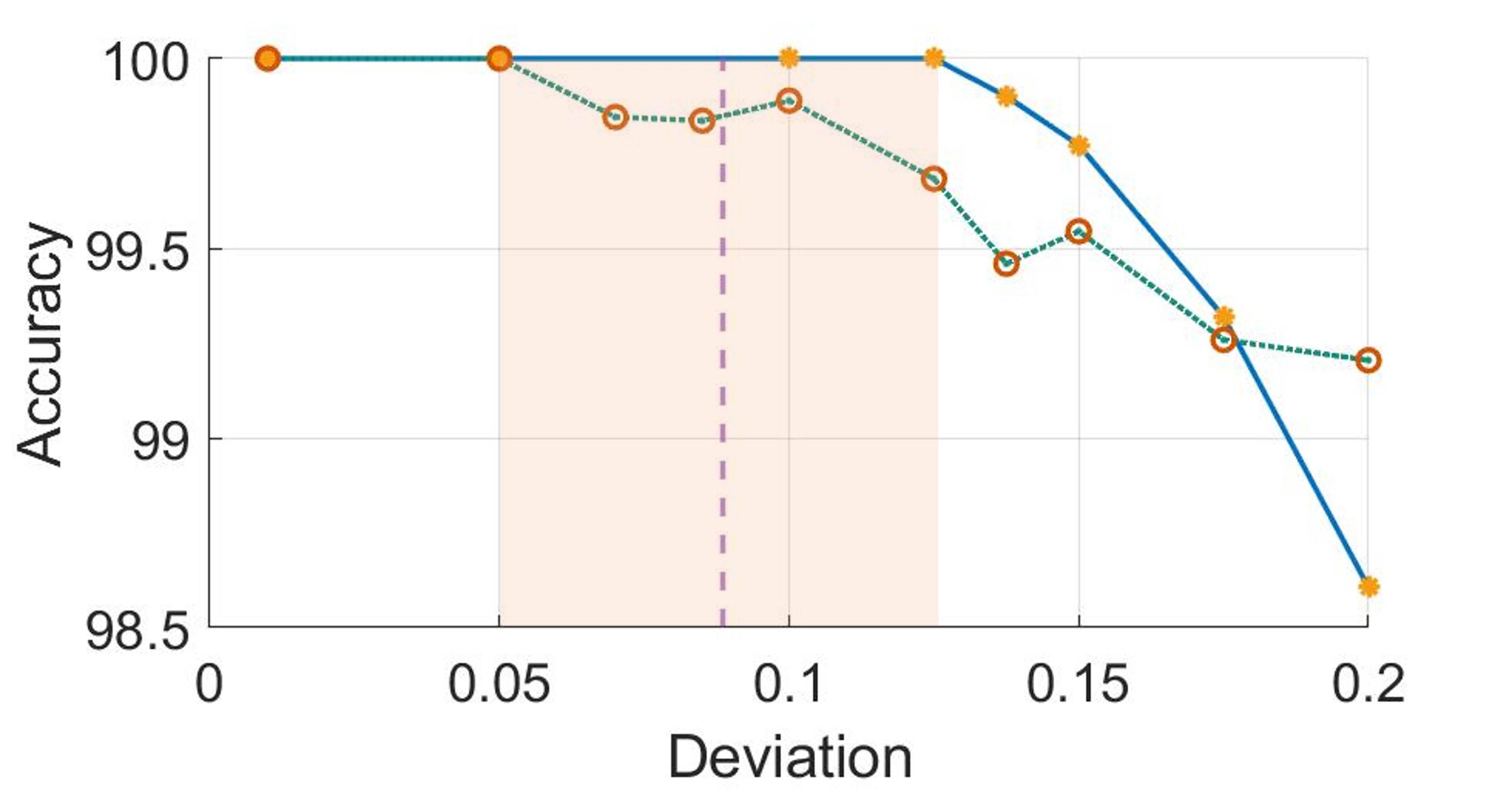}
    \end{subfigure}
	\caption{This figure depicts the classification accuracy when the constant perturbation has been added to the nominal sequences in the $5$'th class in the MNIST dataset with $\varepsilon_{5} / \eta = 0.0886$. From left to right are magnified details. The accuracy eventually converges to $10\%$ since the total number of labels is $m = 10$, and a random guess has $1/10$ of the chance to get the correct label.}
    \label{fig:deviation}
\end{figure*}

\begin{table}                           
	\centering	
	\resizebox{\linewidth}{!}
	{ 
		\begin{tabular}{c c c c c c}
			\toprule
			Stability  & \multirow{2}{*}{$\tau$} & Unperturbed  & Perturbed & Performance  \\
			Constraint & &  Sequence &  Sequence & Loss\\
			\midrule
			No       & - & 95.88 & 94.56 & \textbf{1.3163}\\ [0.5ex]
			\hdashline\noalign{\vskip 0.5ex}
			Yes      & 1.05 & 94.68 & 94.34 & 0.3396   \\
			Yes      & 1.1 & 95.24 & \textbf{94.90} & 0.3416   \\
			Yes      & 2 & 95.50 & \textbf{95.05} & 0.4467   \\
			Yes      & 4 & 94.40 & 94.20 & 0.1978   \\
			Yes      & 8 & 94.14 & 94.04 & 0.1062   \\
			\bottomrule
	\end{tabular} }                        
	\caption{The performance loss under perturbation with various scaling factor $\tau$ on the MNIST dataset.}
	\label{table:robust}  
\end{table}

\begin{figure*}[t]
    \centering
	\includegraphics[width=\linewidth]{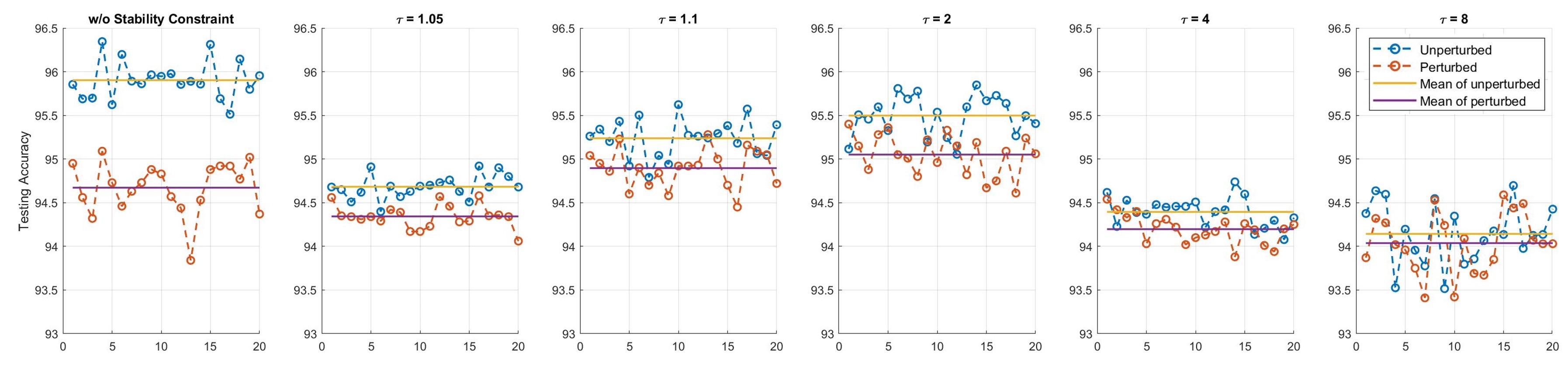}
	\caption{The above figures depict performance with both unperturbed and perturbed data for all models trained with various scaling factor $\tau$ on the MNIST dataset. Each data point represents the average result from $20000$ sampled data sequences.}
    \label{fig:perturb}
\end{figure*}

\subsection{Campus Map Dataset}
In the second case study, we test our model on the Campus Map dataset \cite{liu2021classification}. The training and testing are performed in the same platform with the first case study with a learning rate $l_r = 0.00001$. The testing result is presented in Table \ref{table:accuracy} and Fig. \ref{fig:loss_campus}, which is validated over five random seeds. It is shown that the stable model obtains a comparable performance as the unstable model on the Campus map dataset and enjoys the robustness guarantee when the input perturbation satisfies certain conditions. It should also be emphasized that the performance of the map classification will reach $97\%$ by using a team of communicating robots \cite{liu2020distributed}. 

We also successfully performed a real-world experiment of map classification with aerial robots. Some snapshots taken from the experiments are shown in Fig. \ref{fig:step}, and a full experiment video can be found at \url{https://youtu.be/nsnPFAvJLoY}.

\begin{figure}
    \centering
	\includegraphics[width=\linewidth]{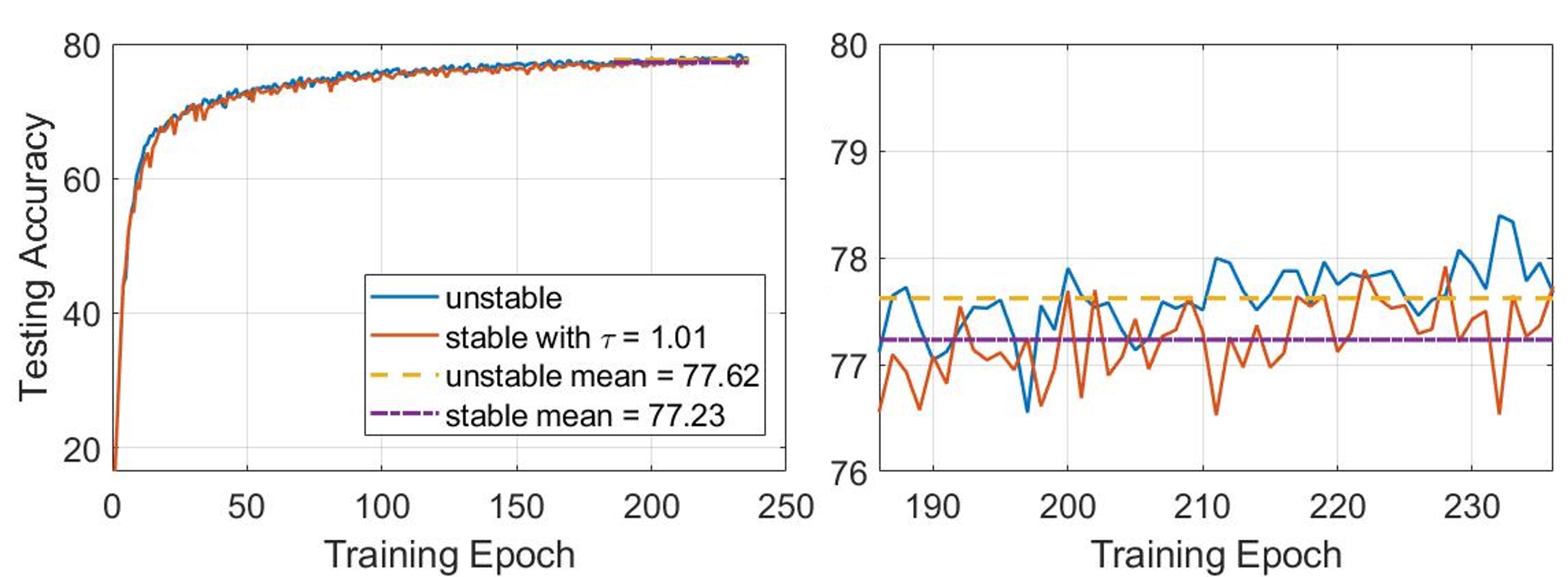}
	\caption{This figure depicts the training progress with and without the stability constraint \eqref{eq:constraint} on the Campus map dataset.}
    \label{fig:loss_campus}
\end{figure}

\begin{figure*}
    \centering
	\includegraphics[width=\linewidth]{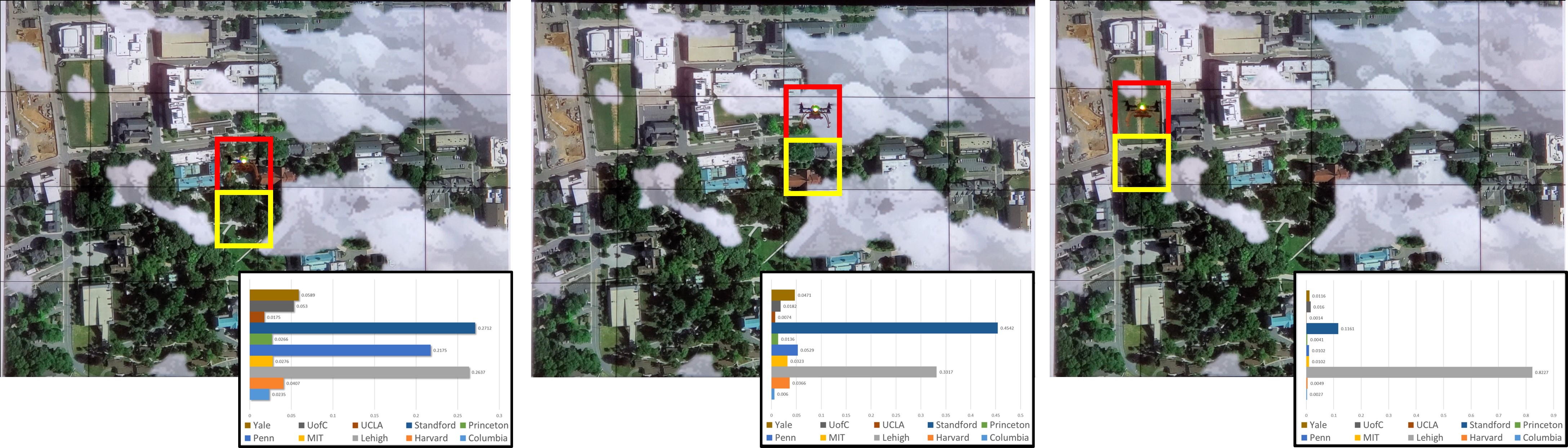}
	\caption{The above figures are snapshots taken while the aerial robot is flying over the map of the Lehigh University campus. The robot's location and observation are denoted with the red and the yellow box. The bar plot shows the real-time belief vector. Eventually, the robot correctly classifies the map as the Lehigh University (gray bar). For more detail, please see the full experiment video at \url{https://youtu.be/nsnPFAvJLoY}.}
    \label{fig:step}
\end{figure*}


\section{Conclusion} \label{sec:conclusion}

We present a framework to analyze the robustness properties of stable RNNs with sequential inputs for classification purposes. It is shown that every trained RNN exhibits robust classification with respect to some bounded perturbations. We quantify robustness bounds in terms of trainable weight matrices. Our results are significant as they reveal interplay among various design (trainable) parameters. Our extensive simulations and one real-world experiment support and validate the usefulness of our theoretical findings.

\printbibliography
\end{document}